\documentclass[letterpaper, 10pt, conference]{ieeeconf}  % Comment this line out
\IEEEoverridecommandlockouts                              % This command is only
\RequirePackage{algorithm}
\usepackage[noend]{algpseudocode}
\usepackage{microtype}
\usepackage{graphicx}
\usepackage{xspace} 
\usepackage{subcaption}
\usepackage[]{caption}

\usepackage{amsmath}
\usepackage{amsfonts}
\usepackage{amssymb}

\usepackage{centernot}
\usepackage{mathtools}
\usepackage{twoopt}
\usepackage{booktabs} % for professional tables
\usepackage{bm}
\usepackage[bibliography=common]{apxproof}

\usepackage[scr=boondoxo,scrscaled=1.05]{mathalfa}

\let\proof\relax
\let\endproof\relax
\usepackage{amsthm}

\usepackage{tikz, tikzscale}
\usetikzlibrary{shapes,arrows,calc,decorations.pathreplacing,calligraphy,cd,positioning,shapes.geometric}

\usepackage{hyperref}
\newcommand{\R}{{\mathbb R}}
\newcommand{\Z}{{\mathbb Z}}
\newcommand{\N}{{\mathbb N}}
\newcommand{\tth}{^{\text{th}}}
\newcommand{\sdots}{\,..\,}

\newcommand{\defeq}{\vcentcolon=}
\newcommand{\zf}{\mathbf{0}}

\newcommand{\cp}{CartPole\xspace}
\newcommand{\re}{Reacher\xspace}

\newcommand{\masked}{\textsc{Masked}\xspace}
\newcommand{\bc}{\textsc{BcVanilla}\xspace}
\newcommand{\bcd}{\textsc{BcManual}\xspace}

\newcommand{\net}{f_{\theta}}
\newcommand{\compnet}{g_{\theta}}
\newcommand{\enc}{\psi_e}
\newcommand{\dec}{\psi_d}

\newcommand{\vaemask}{\tilde \psi}

\newcommandtwoopt{\St}[2][t][]{{S_{#1}^{#2}}} % State
\newcommandtwoopt{\st}[2][t][]{{s_{#1}^{#2}}}
\renewcommandtwoopt{\Im}[2][t][]{{I_{#1}^{#2}}} % Image 
\newcommandtwoopt{\im}[2][t][]{{\scriptstyle I}_{#1}^{#2}}
\newcommandtwoopt{\Ob}[2][t][]{{O_{#1}^{#2}}} % Observation
\newcommandtwoopt{\ob}[2][t][]{o_{#1}^{#2}}
\newcommandtwoopt{\Ac}[2][t][]{{A_{#1}^{#2}}} % Action
\newcommandtwoopt{\ac}[2][t][]{{a_{#1}^{#2}}}
\newcommand{\Hi}[1][t]{{W}} % History
\newcommand{\hi}[1][t]{{w}}

\newcommand{\traj}{{^{i}\bm{\tau}}}
\newcommand{\trajs}{{^{(1..N)}\bm{\tau}}}

\newcommand{\stspace}{{\mathcal{S}}}
\newcommand{\imspace}{{\mathcal{I}}}
\newcommand{\obspace}{{\mathcal{O}}}
\newcommand{\acspace}{{\mathcal{A}}}

\newcommand{\stdim}{{d_{\stspace}}}
\newcommand{\imdim}{{d_{\imspace}}}
\newcommand{\obdim}{{d_{\obspace}}}
\newcommand{\acdim}{{d_{\acspace}}}

\newcommand{\si}{\mathscr{s}}
\newcommand{\oi}{\mathscr{o}}
\newcommand{\ai}{\mathscr{a}}

\newcommand{\Dfull}{{^{(t,t')} D_{\si,\ai}^{\oi}}}

\newcommand{\mask}{m} % Mask on unintervened system (ONLY USE IN THEORY)
\newcommand{\maski}{\widetilde{m}} % Mask under interventional system (USE IN EXPERIMENTAL SECTION, THIS IS WHAT WE ACTUALLY EVALUATE)

\newcommand{\Pd}[1]{P\big(#1\big)}  % Probability distribution / measure, takes argument
\newcommand{\Pdi}[1]{P(#1)} % Inline version
\newcommand{\Pnorm}[1]{\Big\| #1 \Big\|_1} % L_1 norm of probability measure
\newcommand{\Pnormi}[1]{\| #1 \|_1}  % Inline version

\newcommand{\Pdint}[1]{\tilde{P}(#1)} % Interventional probability distribution / measure

\newcommand{\cU}{\mathcal U} % Caligraphic u, for uniform distribution

\newcommand{\doc}{{\mathrm{do}}} % Do calculus intervention
\newcommand{\cau}{\rightarrow} % Causal arrow in true causal graph
\newcommand{\pcau}{\dashrightarrow} % Potential cause detected by algorithm
\newcommand{\npcau}{\not\dashrightarrow} % No potential cause detected by algorithm

\newcommand{\PI}{\mathrel{\perp\mspace{-10mu}\perp}} % Probabilistic independence
\newcommand{\nPI}{\centernot{\PI}} % Not probabilistically independent (i.e. dependent)

\newcommand{\G}{\mathcal{G}} % Causal graph
\newcommand{\Gi}{\widetilde{\mathcal{G}}} % Interventional causal graph
\newcommand{\M}{\mathcal{M}} % Structural causal model
\newcommand{\Mi}{\widetilde{\mathcal{M}}} % Interventional structural causal model

\newcommand{\Gs}{\mathcal{G}_s} % Causal graph for our system
\newcommand{\Gsi}{\widetilde{\mathcal{G}}_s} % Interventional causal graph for our system
\newcommand{\Ms}{\mathcal{M}_s} % Structural causal model for our system
\newcommand{\Msi}{\widetilde{\mathcal{M}}_s} % Interventional structural causal model for our system

\newcommand{\sys}{\langle \Ms, \Gs \rangle}
\newcommand{\sysi}{\langle \Msi, \Gsi \rangle}

\newcommand{\pa}[1]{\mathbf{pa}_{#1}} % Parent set of node

\newcommand{\bV}{\mathbf V} % Set of nodes variables in SCM
\newcommand{\bv}{\mathbf v} % Set of node values in SCM
\newcommand{\bU}{\mathbf U} % Set of noise variables (endogenous
\newcommand{\cF}{\mathcal F} % Causal functions SCM
\newcommand{\SCM}{\langle \bV, \bU, \cF \rangle} % Shorthand for SCM tuple

\newcommand{\bX}{\mathbf X} % Arbitrary subset of nodes in SCM
\newcommand{\bx}{\mathbf x} % Corresponding values to nodes

\newcommand{\bY}{\mathbf Y} % Arbitrary subset of nodes in SCM
\newcommand{\by}{\mathbf y} % Corresponding values to nodes

\newcommand{\bZ}{\mathbf Z} % Arbitrary subset of nodes in SCM
\newcommand{\bz}{\mathbf z} % Corresponding values to nodes

\newcommand{\xt}{x_t}
\newcommand{\dxt}{\dot{x}_t}

\newcommand{\tht}{\theta_t}
\newcommand{\dtht}{\dot{\theta}_t}

\newcommand{\Starget}{S_\mathrm{target}}
\newcommand{\Unif}{\mathrm{Unif}}
\theoremstyle{plain}
\newtheoremrep{theorem}{Theorem}
\newtheoremrep{proposition}[theorem]{Proposition}
\newtheoremrep{lemma}[theorem]{Lemma}
\newtheoremrep{corollary}[theorem]{Corollary}

\theoremstyle{definition}

\newtheorem{assumption}{Assumption}

\theoremstyle{remark}
\newtheorem{remark}{Remark}

\newenvironment{customproofsketch}{\indent \indent \textit{Proof sketch (informal): }}{\hfill $\blacksquare$ \vspace{0.2cm}}

\usepackage[backend=biber,style=numeric,sorting=none]{biblatex}
\renewcommand{\appendixsectionformat}[2]{Proofs for Section~#1}

\title{\LARGE \bf
    Initial State Interventions for Deconfounded Imitation Learning
}

\author{Samuel Pfrommer, Yatong Bai, Hyunin Lee, Somayeh Sojoudi
    \thanks{All authors are with the Department of Electrical Engineering and Computer Sciences, University of California Berkeley, Berkeley, CA, 94720.
    {\tt\small sam.pfrommer@berkeley.edu}; {\tt\small yatong\_bai@berkeley.edu}; {\tt\small hyunin@berkeley.edu}; {\tt\small sojoudi@berkeley.edu}}
}

\begin{document}

\maketitle
\thispagestyle{empty}
\pagestyle{empty}

%%%%%%%%%%%%%%%%%%%%%%%%%%%%%%%%%%%%%%%%%%%%%%%%%%%%%%%%%%%%%%%%%%%%%%%%%%%%%%%%
\begin{abstract}
    Imitation learning suffers from \emph{causal confusion}. This phenomenon occurs when learned policies attend to features that do not causally influence the expert actions but are instead spuriously correlated. Causally confused agents produce low open-loop supervised loss but poor closed-loop performance upon deployment. We consider the problem of masking observed confounders in a disentangled representation of the observation space. Our novel masking algorithm leverages the usual ability to intervene in the initial system state, avoiding any requirement involving expert querying, expert reward functions, or causal graph specification. Under certain assumptions, we theoretically prove that this algorithm is \emph{conservative} in the sense that it does not incorrectly mask observations that causally influence the expert; furthermore, intervening on the initial state serves to strictly reduce excess conservatism. The masking algorithm is applied to behavior cloning for two illustrative control systems: \cp and \re.
\end{abstract}

%%%%%%%%%%%%%%%%%%%%%%%%%%%%%%%%%%%%%%%%%%%%%%%%%%%%%%%%%%%%%%%%%%%%%%%%%%%%%%%%
\let\thefootnote\relax\footnotetext{The full technical report, including proofs, can be found on \href{https://arxiv.org/abs/2307.15980}{arXiv}.}

\section{INTRODUCTION}
Imitation learning aims to train an intelligent agent to mimic expert demonstrations for a particular task. Various imitation learning instantiations, such as behavior cloning and inverse reinforcement learning, have been widely applied to fields including robotics \cite{Calinon2007IncrementalLO, Krishnan2018SWIRLAS}, autonomous driving \cite{Wang2021InverseRL, Kuefler2017ImitatingDB}, and optimal navigation \cite{hussein2018deep, shou2020optimal}. Imitation learning enables agents to learn from high-quality samples instead of exploring from scratch, leading to significantly higher learning efficiency when compared with reinforcement learning methods \cite{bojarski2016end}. This is especially important in safety-critical settings where reinforcement learning are difficult to execute \cite{yin2021imitation, pfrommer2022safe}. Even when the flexibility of reinforcement learning is desired, imitation learning can be used to accelerate the learning process \cite{Hester2017DeepQF}.

Despite its broad applicability, imitation learning exhibits an issue known as \textit{causal confusion} \cite{de2019causal}: the learned policy misattributes features which are primarily \emph{correlated} with expert actions as reflecting a \emph{causal} relationship \cite{kaddour2022causal}. This can manifest itself both through the observed features which are spuriously correlated with the expert actions (``nuisance variables'') as well as confounders which are available to the expert but not the imitator (``unobserved confounders''). We restrict ourselves to the former, although for completeness we include approaches addressing the latter in our work.

Consider an illustrative example of causal confusion adapted from \cite{de2019causal}. The task at hand is learning to drive a car from expert demonstrations. A behavior cloning agent is provided video observations from the driver's perspective, including a brake light on the dashboard. Although the learned braking policy is excellent on the supervised dataset, deployment performance is poor: the agent has effectively learned to brake when the brake light is on, instead of attending to other pedestrians or vehicles. In this case, the brake light is a ``nuisance variable,'' and we can dramatically improve the performance of the policy by covering the brake light and reducing information for the model.

Existing approaches for completely masking such nuisance variables generally require either a queryable expert or access to the expert reward function. The seminal work of \cite{de2019causal} introduced a $\beta$-Variational Auto Encoder ($\beta$-VAE) decomposition the observation space along with a joint policy parameterized by hypothetical causal structures. The space of causal structures can then be searched with two distinct algorithms, one leveraging expert queries and the other based on policy evaluations and reward feedback. The existence of nuisance variables was also noted \cite{ortega2021shaking} as part of a broader issue with sequential models that can be addressed with Dagger-style expert queries \cite{ross2011reduction}. The work of \cite{park2021object} partially addresses the nuisance variable problem by regularizing the learned policies to attend to multiple objects in the scene. While this approach does not require policy executions, it only weakens the learner's attention to a nuisance variable and does not eliminate it completely.

The complementary problem of unobserved confounders considers the setting where experts observe confounding variables that are inaccessible to the learner. In the car driving example, this might include a human driver listening to honking that is not detected with visual sensors. One exciting theoretical line of research in this area \cite{zhang2020causal, kumor2021sequential} presents causal-model derived conditions for imitability and an algorithm for imitating the expert policy when possible. However, these works make the strong assumption that the causal graph is provided to the imitation learning agent. Other efforts to apply causal inference techniques to the unobserved confounder problem either require strong assumptions, such as the knowledge of the expert reward \cite{etesami2020causal} and purely additive temporally correlated noise \cite{swamy2022causal}, or only evaluate simple multi-armed bandit problems \cite{vuorio2022deconfounded}.

This work focuses on the problem of observed nuisance variables. Our approach, presented in Section~\ref{sec: method} leverages initial state interventions to identify and completely mask causally confusing features without relying on expert queries or policy interventions. We provide \emph{conservativeness} guarantees for our method in Section~\ref{sec: theory} and present illustrative experiments in Section~\ref{sec: experiments}.

\section{NOTATION AND BACKGROUND}

We denote the set of real numbers by $\R$ and the set of natural numbers by $\N$. The set $\{1, \dots, a\} \subset \N$ is denoted by $[a]$ for $a \in \N$, and similarly ${a, \dots, b} \subset \N$ is denoted by $[a \sdots b]$. For a pair of boolean variables $x$ and $y$, the notation $\wedge$ denotes the ``and'' operator while $\vee$ denotes ``or.'' For a set of boolean variables $\{ x_1, x_2, \ldots, x_n \}$, the notations $\bigwedge_{i=1}^n x_i$ and $\bigvee_{i=1}^n x_i$ denote $x_1 \wedge x_2 \wedge \ldots \wedge x_n$ and $x_1 \vee x_2 \vee \ldots \vee x_n$, respectively. The logical negation of a boolean variable or vector $x$ is denoted by $\neg x$. We denote the identically zero function on a domain by $\zf$, and we write $f(\cdot) \not \equiv \zf$ to mean that $f(\cdot)$ is not equivalent to the zero function over its argument---i.e., there exists an input where $f$ is nonzero.

\subsection{Measure theory and probability}

For a random variable $X$, we introduce the notation ${\Pdi{x} \in M(X)}$ to represent a probability measure over the values $x$ in the domain of $X$, contained in the space of measures $M(X)$. The uniform measure over an interval $[a,b] \subset \R$ is denoted by $\cU(a,b)$. For two measures $\mu$ and $\nu$, we say that $\nu$ is absolutely continuous with respect to $\mu$ if for every $\mu$-measurable set $A$, $\mu(A) = 0$ implies $\nu(A) = 0$. If $\nu$ is absolutely continuous with respect to $\mu$, we let $d\nu / d\mu$ denote the Radon-Nikodym derivative of $\nu$ with respect to $\mu$. The standard Lebesgue measure on $\R$ is denoted $\lambda$. For a measure $\mu$ which is absolutely continuous with respect to $\lambda$, we define its $L_1$ norm in the typical manner
\begin{align*}
    \Pnormi{\mu} \defeq \int \biggl \lvert \frac{d \mu}{d \lambda} \biggr \rvert d \lambda,
\end{align*}
which we take to be the default norm in the Banach space of measures on $\R$. We denote independence between two random variables using $\PI$ and its negation by $\nPI$.

\subsection{Causal graphs and structural causal models}

We denote a directed acyclic graph by $\G$, with the presence of a direct edge between nodes $X$ and $Y$ denoted $X \cau Y$. For a given node $X$ in $\G$, we let $\G_{\underline{X}}$ denote the graph obtained by deleting outgoing edges from $X$. We denote sets of nodes in a graph using bold font (e.g., $\bZ$). The set of parents of a node $X$ in a graph is denoted by $\pa{X}$. A path between two nodes $X$ and $Y$ can consist of arbitrarily directed edges and is said to be blocked by a set of nodes $\bZ$ if the path contains any of the following \cite{Pearl09}:
\begin{itemize}
    \item A chain $I \cau M \cau J$ with $M \in \bZ$.
    \item A fork $I \leftarrow M \cau J$ with $M \in \bZ$.
    \item A collider $I \cau M \leftarrow J$ such that $M \not \in \bZ$ and no descendant of $M$ is in $\bZ$.
\end{itemize}
     Two nodes $X$ and $Y$ are said to be d-separated by $\bZ$ if $\bZ$ blocks every path between $X$ and $Y$. We call a path with all edges oriented the same direction a directed path.

We leverage Pearl's structural causal model (SCM) formalism \cite{Pearl09}. An SCM $\M = \SCM$ consists of endogenous variables $\bV$, exogenous variables $\bU$, and structural equations $\cF$. Each $V \in \bV$ is represented by a node in the causal graph $\G$ and associated with an independently distributed exogenous variable $U_V \in \bU$. The structural equations $f_V \in \cF$ assign values of a particular node $V \in \bV$ as a function $V \defeq f_V(\pa{V}, U_V)$ of its parents and associated exogenous variable. The SCM $\M$ induces a joint distribution $\Pd{\bv}$ over the endogenous variables $\bV$. We say that an SCM $\M$ is faithful to its causal graph $\G$ if the distribution $\Pd{\bv}$ induced by $\M$ contains only the pairwise conditional independencies implied by $\G$; i.e. $X \PI Y \mid \bZ$ in the joint distribution from $\M$ iff $X$ and $Y$ are d-separated by $\bZ$ in $\G$ \cite{Spirtes2000}. As a notable special case, if $\bZ$ is empty and there exists a path from $X$ to $Y$ with no colliders then $X \nPI Y$. 

We define an intervention on a particular node $V$ to be a reassignment of the associated structural equation $f_V$. This intervention can take the form of a constant intervention $V \defeq v$, which we denote by $\doc(V = v)$ for a constant $v$ and may abbreviate to $\doc(v)$. We also define a distributional intervention, denoted by $\doc(V \sim \Pdint{v})$, where we assign $V$ to be drawn from a specified distribution $\Pdint{v}$. We denote the post-intervention SCM by $\Mi$, with an associated causal graph $\Gi$ identical to $\G$ but with incoming edges to $V$ removed. Note that reassigning the associated structural equation for any particular node $V$ induces a new distribution generated by $\Mi$ over the set of all endogenous variables $\bV$, which we denote by $\Pdi{\bv \mid \doc(V = v)}$ or $\Pdi{\bv \mid \doc(V \sim \Pdint{v})}$.

\subsection{Behavior cloning} \label{sec: imitation_learning}
Behavior cloning uses expert trajectories to train an imitating policy. For the system of interest, we use $\stdim$, $\imdim$, $\obdim$, and $\acdim$ to denote the dimensionality of the bounded state space $\stspace \subseteq \R^{\stdim}$, raw image observation space $\imspace \subseteq \R^{\imdim}$, disentangled observation space $\obspace \subseteq \R^{\obdim}$, and action space $\acspace \subseteq \R^{\acdim}$. Let $\St$, $\Im$, $\Ob$, and $\Ac$ be vector random variables taking on values in $\stspace$, $\imspace$, $\obspace$, and $\acspace$, respectively, for a discrete time step $t \in \N$. States variables $\St$ represent the intrinsic low-dimensional dynamics of the system (e.g. simulator variables) while observations $\Ob$ are distilled using a VAE-style framework from high-dimensional image measurements $\Im$, with $\imdim \gg \obdim$. The system dynamics assume that $\St[t+1]$ is strictly a function of $\St$ and $\Ac$. Lower-case script letters $\si \in [\stdim]$, $\oi \in [\obdim]$, and $\ai \in [\acdim]$ denote specific indices in the state, observation, and action vectors. For example, $\St[1][\si]$ refers to the real-valued random variable corresponding to the $\si \tth$ state variable at the first time step. We model $\Hi[1] \sim \cU(a,b)$ to be an unobserved variable capturing uncontrolled and unknown initialization stochasticity (i.e. a random ``seed'').

The collection of states, observations, and actions, along with $\Hi[1]$, comprise endogenous variables in an SCM defining our system. We denote the system SCM by $\Ms$ and denote the corresponding faithful causal graph by $\Gs$. Note that the SCM depends on the choice of policy. Since we aim to infer causalities regarding the expert policy, we generally let any causal relationships refer to the $\Ms$ and $\Gs$ induced by the expert policy unless otherwise stated. We pair the system SCM and causal graph with the tuple $\sys$. Although nodes in $\Gs$ are individual elements in our vector-valued random variables (i.e., $\St[t][\si]$ is a node, not $\St[t]$), with some abuse of notation, we let the edge symbol $\St[t] \cau X$ signify that $\St[t][\si] \cau X$ for some $\si \in [\stdim]$. Similarly, $X \cau \St[t]$ denotes that $X \cau \St[t][\si]$ for some $\si$.

\begin{figure}
    \hspace*{1cm}\begin{tikzpicture}
    \tikzstyle{var}=[inner sep=0.0cm, minimum size=0.75cm]
    \tikzstyle{cause}=[thick, ->, shorten >= 4pt]
    \newcommand\septop{-0.1cm}
    \newcommand\sepbot{0.3cm}

    \newcommand\cola{0cm}
    \newcommand\colb{2cm}
    \newcommand\colc{4cm}

    \newcommand\rowa{0cm}
    \newcommand\rowb{-3.5cm}

    \node[var, blue] (st11) at (\cola, \rowa) {$\St[1][1]$};
    \node[var, blue, below = \septop of st11] (st12) {$\St[1][2]$};
    \node[var, blue, below = \sepbot of st12] (st1f) {$\St[1][\stdim]$};
    \node at ($(st12)!.45!(st1f)$) {\vdots};

    \node[var, red] (ob11) at (\colb, \rowa) {$\Ob[1][1]$};
    \node[var, below = \septop of ob11] (ob12) {$\Ob[1][2]$};
    \node[var, below = \sepbot of ob12] (ob1f) {$\Ob[1][\obdim]$};
    \node (ob1dots) at ($(ob12)!.45!(ob1f)$) {\vdots};

    \node[var] (ac11) at (\colc, \rowa) {$\Ac[1][1]$};
    \node[var, below = \septop of ac11] (ac12) {$\Ac[1][2]$};
    \node[var, below = \sepbot of ac12] (ac1f) {$\Ac[1][\acdim]$};
    \node at ($(ac12)!.45!(ac1f)$) {\vdots};

    \node[var] (st21) at (\cola, \rowb) {$\St[2][1]$};
    \node[var, below = \septop of st21] (st22) {$\St[2][2]$};
    \node[var, below = \sepbot of st22] (st2f) {$\St[2][\stdim]$};
    \node at ($(st22)!.45!(st2f)$) {\vdots};

    \node[var, red] (ob21) at (\colb, \rowb) {$\Ob[2][1]$};
    \node[var, below = \septop of ob21] (ob22) {$\Ob[2][2]$};
    \node[var, below = \sepbot of ob22] (ob2f) {$\Ob[2][\obdim]$};
    \node (ob2dots) at ($(ob22)!.45!(ob2f)$) {\vdots};

    \node[var] (ac21) at (\colc, \rowb) {$\Ac[2][1]$};
    \node[var, below = \septop of ac21] (ac22) {$\Ac[2][2]$};
    \node[var, below = \sepbot of ac22] (ac2f) {$\Ac[2][\acdim]$};
    \node at ($(ac22)!.45!(ac2f)$) {\vdots};

    \node[var, above=0.5cm of ob11] (w) {$\Hi[1]$};

    \node[below = 1.0cm of ob22] {\vdots};

    \draw [decorate, decoration = {brace}, thick] (-0.5, -2) -- (-0.5, 0.2) node[pos=0.5,left=10pt,rotate=90,xshift=0.55cm]{$t=1$};
    \draw [decorate, decoration = {brace}, thick] (-0.5, -5.5) -- (-0.5, -3.3) node[pos=0.5,left=10pt,rotate=90,xshift=0.55cm]{$t=2$};
    
    % W causes
    \draw [cause, shorten >= -4pt] (w.south) -- (ob11.north);
    \draw [cause, shorten >= -4pt] (w.south) -- (ac11.north west);
    \draw [cause, shorten >= -4pt, opacity=0.1] (w.south) -- (st12.north east);
    
    % T=1 causes
    \draw [cause] (st11.east) -- (ob12.west);
    \draw [cause] (st12.east) -- (ob12.west);
    \draw [cause, shorten >= 10pt] (st12.east) -- (ob1dots.west);
    \draw [cause] (ob1f.east) -- (ac12.west);
    \draw [cause] (st1f.east) -- (ob1f.west);
    \draw [cause] (ob12.east) -- (ac11.west);
    \draw [cause] (ob12.east) -- (ac1f.west);

    % T=2 causes
    \draw [cause] (st21.east) -- (ob22.west);
    \draw [cause] (st22.east) -- (ob22.west);
    \draw [cause, shorten >= 10pt] (st22.east) -- (ob2dots.west);
    \draw [cause] (ob2f.east) -- (ac22.west);
    \draw [cause] (st1f.east) -- (ob1f.west);
    \draw [cause] (st2f.east) -- (ob2f.west);
    \draw [cause] (ob22.east) -- (ac21.west);
    \draw [cause] (ob22.east) -- (ac2f.west);

    \draw [cause, shorten >= 0pt] (ob1f.south) -- (ac21.north);
    
    \tikzstyle{conbot}=[above=.1cm,inner sep=0]
    \tikzstyle{contop}=[below=.1cm,inner sep=0]

    \draw [decorate, decoration = {brace}, thick] (st1f.south east) -- (st1f.south west) node[pos=0.5,contop] (st1) {};
    \draw [decorate, decoration = {brace}, thick] (ac1f.south east) -- (ac1f.south west) node[pos=0.5,contop] (ac1) {};
    % \draw [decorate, decoration = {brace}, thick] (ob1f.south east) -- (ob1f.south west) node[pos=0.5,contop] (ob1) {};

    \draw [decorate, decoration = {brace}, thick] (st21.north west) -- (st21.north east) node[pos=0.5,conbot] (st2) {};

    \draw [cause, shorten >= 0pt] (st1) -- (st2);
    \draw [cause, shorten >= 0pt] (ac1) -- (st2);
    \draw [cause, shorten >= 0pt] (ac1) -- (ob21);
\end{tikzpicture}
    \caption{An example (unknown) system causal graph $\Gs$. We hope to mask $\Ob[][1]$ (e.g. brake light observation), which has no causal edge to any expert action but is correlated with $\Ac[][1]$ through the confounding random ``seed'' $\Hi[1]$ and future spurious correlations. In $\Gs$, $\Hi[1]$ also causally influences $\St[1][2]$; however, if we intervene on $\St[1]$ (blue) this edge is removed in $\Gsi$ (light shading). This enables our masking algorithm to more reliably leverage state initialization to detect potential causes between observations and actions (Section~\ref{sec: derivation}).}
    \label{fig: structure}
\end{figure}
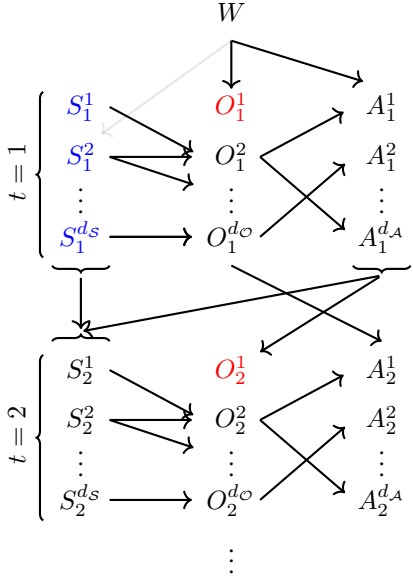

This work evaluates the importance of interventionally assigning the initial state to a particular distribution ${\St[1] \sim \Pdint{\st[1]}}$. This intervention yields a modified SCM $\Msi$ with a corresponding (not necessarily faithful) causal graph $\Gsi$, which removes the edge $\Hi[1] \to \St[1]$ in $\Gs$ (Figure~\ref{fig: structure}). We collect $N$ arbitrary-length expert trajectories from $\Msi$. The collection of all such trajectories is denoted $\trajs$. Among these $N$ trajectories, the $i^{\textrm{th}}$ trajectory consists of the tuple
\[
	\traj = \langle \st[1], \dots, \st[T];\ \im[1], \dots, \im[T];\ \ob[1], \dots, \ob[T];\ \ac[1], \dots, \ac[T] \rangle,
\]
where lowercase letters represent a concrete random variable value (to avoid confusion with indices, we use $\im$ to denote a value of $\Im$). Implicit in this definition is the existence of an \emph{encoder} $\enc: \imspace \to \obspace$ mapping each image $\im$ to a disentangled observation $\ob$. We characterize trajectories as containing observations for simplicity; our environment only provides the images $\im$, and the extraction of disentangled observations $\ob$ is method-dependent.

When training agents on $\trajs$, we parameterize policies as a neural network $\net: \imspace^L \to \acspace$. The neural policy maps some history of observations to an action $a_t$ via
\begin{align} \label{eqn: policy}
    \ac[t] = \net(\im[t], \im[t-1], \dots, \im[t-L+1]).
\end{align}
We then train $\net$ via standard behavior cloning by randomly sampling batches of images and expert actions from $\trajs$ and performing supervised regression.

\subsection{Statistical independence tests} \label{sec: hoeffding}
\newcommand{\Nhoeff}{N_{\textrm{Hoeff}}}

Our method relies on identifying whether two random variables are statistically dependent. While this is a challenging problem with a rich literature \cite{sheskin2020handbook}, in this paper, we only briefly introduce a well-known independence test for continuous distributions based on Hoeffding's D statistic \cite{hoeffding1994non, even2020independence}. Consider two real-valued random variables $X$ and $Y$ with a joint cumulative distribution function ${F(x, y) = \Pd{X \leq x, Y \leq y}}$. Hoeffding's D statistic operates on $\Nhoeff$ independent pairs of observations $\{(X_1, Y_1), \dots (X_{\Nhoeff}, Y_{\Nhoeff})\}$ and outputs a real number $D$ in the range $[-0.5, 1]$, with $D > 0$ indicating dependence. The computational complexity of calculating this statistic is $\mathcal{O} (\Nhoeff \log \Nhoeff)$. For absolutely continuous joint distributions, the D statistic is unbiased and consistent as $\Nhoeff \to \infty$, meaning that the dependence is correctly represented with probability arbitrarily close to $1$. Subsequent variations of the D statistic maintain consistency even for non-absolutely continuous joint distributions \cite{blum1961distribution}, although these complications are outside the scope of our work. We refer to the independence test based on the Hoeffding's D statistic as Hoeffding's independence test.

\section{PROBLEM STATEMENT AND METHOD} \label{sec: method}
We address the \emph{causal confusion} problem in imitation learning and aim to mask spuriously correlated observations. To this end, we investigate the following problem statement:

\begin{center}
\emph{How can we identify and eliminate spuriously correlated observations without relying on online expert queries or knowledge of the expert reward function?}
\end{center}

Our approach addresses this problem in a theoretically grounded way. Specifically, we make the following contributions:
\begin{enumerate}
    \item We present an algorithm for identifying and masking causally confusing observations \emph{without relying on reward function knowledge, expert queries, or causal graph knowledge}.
    \item We prove that, under certain conditions, our procedure is \emph{conservative}: if an observation causally affects the expert actions, it will not be masked.
    \item We demonstrate the importance of \emph{initial state interventions} by showing theoretically that the interventions reduce excess conservatism in the masking algorithm.
\end{enumerate}

Section~\ref{sec: assumptions} presents and analyzes the assumptions underlying our method. Section~\ref{sec: derivation} motivates and derives our method, which is then presented formally in Section~\ref{sec: workflow}.

\subsection{Assumptions} \label{sec: assumptions}

Our proposed method relies on the following assumptions to ensure the theoretical guarantees in Section~\ref{sec: theory}.

\begin{assumption} \label{ass: time_invariance}
The system causal graph $\Gs$ is time invariant. Namely, consider two arbitrary time steps ${t, t' \in \N}$ with $t' \geq t$ and two arbitrary time-indexed variables $X_t$ and $Y_{t'}$ in $\Gs$. Then if $X_t \cau Y_{t'}$ is an edge in $\Gs$, then so is $X_{t + \Delta} \cau Y_{t' + \Delta}$ for any $\Delta \in \Z$ such that ${\min(t+\Delta,t'+\Delta) \geq 1}$.
\end{assumption}

Time-invariance of the expert policy allows for causal inference via interventions on the initial state $\St[1]$. Otherwise we would require the ability to intervene at arbitrary time steps, which is unrealistic for most real-world systems.

\begin{assumption} \label{ass: expert_attend}
The expert policy attends only to observational information derived from the underlying state. Namely, if $\Ob[t][\oi] \cau \Ac[t'][\ai]$ in $\Gs$ for $t,t' \in \N$ with $t' \geq t$, then there must exist an index $\si$ such that ${\St[t][\si] \cau \Ob[t][\oi]}$.
\end{assumption}

Assumption~\ref{ass: expert_attend} reflects the intuition that the expert policy itself must not be fooled by spurious information in the observation space. This is a natural assumption in the considered case where the dynamics of the underlying system depend only on $\St$, not $\Ob$. 

\begin{assumption} \label{ass: reaction_horizon}
    The expert policy reacts to observations within a \emph{reaction horizon} $H \in \N$. Specifically, if $\Ob[t][\oi] \cau \Ac[t_1][\ai]$ in $\Gs$ for some $t_1 > t$ and particular $t \in \N$, $\oi \in [\obdim]$, and $\ai \in [\acdim]$, then there exists a $t_2 \in [t \sdots t+H-1]$ such that $\Ob[t][\oi] \cau \Ac[t_2][\ai]$.
\end{assumption}

Assumption~\ref{ass: reaction_horizon} imposes a horizon within which the expert is assumed to react to a hypothetical intervention on a state or observation. For finite-length trajectories, $H$ can be chosen to be the entire trajectory length, with the algorithm and theory still valid. As such, $H$ introduces a hyperparameter that allows for more tractable computation under some assumptions on the expert. Our experiments show that $H$ can be much smaller than the trajectory length for certain practical dynamic systems and experts.

Finally, we formalize a class of SCMs that behave nicely under interventions.

\begin{assumption} \label{ass: abs_cont}
The system SCM $\Ms=\SCM$ is \emph{interventionally absolutely continuous}, meaning that for any disjoint sets of nodes $\bX$, $\bY$, and $\bZ$, the interventional distribution $\Pd{\bz \mid \doc(\bX = \bx), \by}$ is absolutely continuous with respect to the Lebesgue measure, has a bounded Radon-Nikodym derivative, and is continuous as a measure-valued function with respect to $\bx$ and $\by$.
\end{assumption}

Assumption~\ref{ass: abs_cont} stipulates that the probability distribution induced by our SCM on any set of non-intervened nodes is absolutely continuous with bounded density. This is a technical condition that facilitates analysis and allows us to assert that Hoeffding's test is consistent. We note that subsequent D-statistic variations allow for non-absolutely continuous joint distributions \cite{blum1961distribution} --- we leave the theoretical and practical implications of more sophisticated testing to future work.
% \yb{How restrictive is this assumption?}

\subsection{Derivation} \label{sec: derivation}

Our aim is to mask a particular observation $\Ob[][\oi]$ across all time steps if it has no causal effect on any expert action within the reaction horizon. As intervening on observations is impractical, this causality is challenging to deduce. We do, however, assume the ability to intervene on the system in one specific instance: setting the state variables $\St[1]$ at initialization. We manipulate $\St[1]$ to infer the possible existence of a true causal relationship.

We first motivate our approach from an arbitrary time step $t \geq 2$ before specializing on the initialization. Consider arbitrary observation and action indices ${\oi \in [\obdim], \ai \in [\acdim]}$ and time steps $t, t' \in \N$ with $t' \in [t \sdots t + H - 1]$. Assumption~\ref{ass: expert_attend} states that a causal effect $\Ob[t][\oi] \cau \Ac[t'][\ai]$ must arise from a larger causal path
\begin{align} \label{eqn: causal_path}
    \St[t][\si] \cau \Ob[t][\oi] \cau \Ac[t'][\ai]
\end{align}
in $\Gs$, for some state variable index $\si \in [\stdim]$. We now observe that by faithfulness of $\sys$ it must be that $\St[t][\si] \nPI \Ob[t][\oi]$ and $\St[t][\si] \nPI \Ac[t'][\ai]$; i.e. the causal relationships in $\Gs$ imply probabilistic dependencies in the induced distribution from $\Ms$. Note that these are \emph{statistical} statements which can be ascertained from the observational data. We define the boolean variable $\Dfull$ to check these independencies:
    \begin{align} \label{eqn: causal_check}
        \Dfull \coloneqq (\St[t][\si] \nPI \Ob[t][\oi]) \land (\St[t][\si] \nPI \Ac[t'][\ai]),
    \end{align}
    and introduce the ``potential cause'' notation 
    \begin{align} \label{eqn: potential_cause}
        \Ob[t][\oi] \pcau \Ac[t'][\ai] \coloneqq \bigvee_{\si=1}^{\stdim} \left( \Dfull \right).
    \end{align}

The boolean-valued statement $\Ob[t][\oi] \pcau \Ac[t'][\ai]$ intuitively captures that, based on observational data, there may (but need not) exist a true causal edge $\Ob[t][\oi] \cau \Ac[t'][\ai]$ generated by some $\St[t][\si]$ as in \eqref{eqn: causal_path}. We denote by $\Ob[t][\oi] \npcau \Ac[t'][\ai]$ the logical negation of $\Ob[t][\oi] \pcau \Ac[t'][\ai]$. As we will elaborate in more detail shortly, if $\Ob[t][\oi] \npcau \Ac[t'][\ai]$ for all actions $\ai \in [\acdim]$ and $t'$ in the reaction horizon, we want to ``mask'' the $\oi \tth$ observation as it has no causal effect on the expert action but could be spuriously correlated in a way that undermines the imitation learning policy performance. 

It is immediate from the above faithfulness argument that for $t \geq 2$, we have the implication
\begin{align} \label{eqn: conservative_general}
    \Ob[t][\oi] \cau \Ac[t'][\ai] \implies \Ob[t][\oi] \pcau \Ac[t'][\ai].
\end{align}
Note that (\ref{eqn: conservative_general}) provides a \emph{conservativeness} guarantee: if an observation causally influences an action, we will not mistakenly conclude from observational data that it does not, and hence incorrectly mask an observation that is actually used by the expert policy. However, this conservativeness is not apparent for $t=1$ in the modified causal model $\sysi$, where we intervene to specify the initial state distribution, overriding the natural randomness resulting from $\Hi[1]$ and potentially breaking faithfulness. As a simple counterexample, initializing $\St[1]$ to a constant vector would make $\St[1][\si]$ independent of every other random variable in the causal graph, and therefore no potential causes could be discovered as \eqref{eqn: causal_check} would always be false. Nonetheless, when a sufficiently sensible initialization distribution is used, we prove that the conservativeness result still holds under intervention on $\St[1]$ in Section~\ref{sec: theory}.

The reverse implication to \eqref{eqn: conservative_general} does not hold. It is possible that spurious statistical relationships exist while a causal edge $\Ob[t][\oi] \to \Ac[t'][\ai]$ does not. Indeed, for $t \geq 2$, the abundance of chronologically antecedent variables virtually guarantees that all variables have share a common cause and hence a statistical dependence. The sole exception is the initial state $\St[1]$. By intervening on $\St[1]$, we eliminate the incoming edge from the only possible common ancestor $\Hi[1]$ in the causal graph (Figure~\ref{fig: structure}). Therefore, we expect that this interventional ability should help eliminate potential causes $\Ob[1][\oi] \pcau \Ac[t'][\ai]$ which do not exist in the true causal graph and reduce excessive conservativeness in the algorithm. We analyze this idea formally in Section~\ref{sec: theory}.

The culmination of our efforts is described in Algorithm~\ref{alg: mask}, which checks for potential causes at $t = 1$ using expert data $\trajs$ collected from the interventional system $\sysi$. Note that Algorithm~\ref{alg: mask} invokes the $\textsc{Hoeffding}$ routine to compute Hoeffding's D statistic for independence between two variables. This test is computed over our dataset of trajectories $\trajs$, extracting exactly one pair of variables from each trajectory ($\Nhoeff = N$). For concreteness, consider the call $\textsc{Hoeffding}(\St[1][2] \nPI \Ac[3][4]\ \mathrm{in}\ \trajs)$. This extracts, from each trajectory, the second element of the $t=1$ state and the fourth element of the $t=3$ action. These $N$ pairs are then supplied to Hoeffding's test, which returns a real number in the range $[-0.5, 1]$, with a value greater than zero indicating dependence. Since perfect observational disentanglement is unrealistic, we introduce a small positive threshold hyperparameter $\gamma$.

Algorithm~\ref{alg: mask} is presented for readability and can be implemented more efficiently. The Hoeffding tests between $\St[t][\si]$ and $\Ob[t][\oi], \Ac[t'][\ai]$ can be precomputed, yielding the runtime
\[
    O\left(\stdim (\obdim + H \acdim) N \log N\right),
\]
where $N \log N$ is the cost of evaluating Hoeffding's test for a specific pair of variables over $N$ trajectories. In practice, Hoeffding's test executions are very fast---on the order of milliseconds for $N=10^3$---and incur a negligible overhead compared with the training time of imitation learning.

 \begin{remark}
     The reader may have noticed that our approach bears a resemblance to \emph{instrumental variable regression}, a statistical technique for estimating causal relationships that has also received some attention in the causal imitation learning literature \cite{swamy2022causal}. We emphasize that $\St[t][\si]$ does not constitute a valid instrumental variable in the causal path \eqref{eqn: causal_path} as there may be many other paths between $\St[t][\si]$ and $\Ac[t'][\ai]$ which are not mediated by $\Ob[t][\oi]$. Thus while the spirit of our approach is related to instrumental variable regression, we cannot use $\St[t][\si]$ to precisely determine a causal relationship between $\Ob[t][\oi]$ and $\Ac[t'][\ai]$ and only use $\St[t][\si]$ to provide evidence of a potential cause.
\end{remark}

\subsection{Imitation Learning Workflow} \label{sec: workflow}

Drawing on the masking approach developed in Section~\ref{sec: derivation}, we summarize our overall deconfounded imitation learning workflow as the following four steps.

\begin{enumerate}
    \item Collect random-policy trajectories to learn a observation representation using a $\beta$-VAE, denoted by ${\dec\, \circ\, \enc: \imspace \to \imspace}$, with an encoder $\enc: \imspace \to \obspace$ and decoder $\dec: \obspace \to \imspace$. For a well-trained $\beta$-VAE, $\dec\, \circ\, \enc$ approximates the identity. We rely on $\beta$-VAEs' latent space regularization to produce disentangled observations.
    \item Collect a sequence of $N$ trajectories $\trajs$ from the expert policy, with the starting state distribution $\Pdint{\st[1]}$ over $\stspace$ having any density that is everywhere nonzero (e.g. uniform).
    \item Execute Algorithm~\ref{alg: mask} on $\trajs$ to obtain the observation mask $\maski \in \{0,1\}^{\obdim}$, where $\maski_{\oi} = 1$ if the $\oi \tth$ observation is to be masked.
    \item Train the final policy $\compnet: \imspace^L \to \acspace$ on $\trajs$ using standard supervised learning; $\compnet$ masks the disentangled observation space using $\maski$ before executing a learnable policy network $\net$:
        \[
            \compnet(\im[t], \dots, \im[t-L+1]) = \net(\vaemask(\im[t]), \dots, \vaemask(\im[t-L+1])),
        \]
        where the masked $\beta$-VAE $\vaemask: \imspace \to \imspace$ has its weights fixed and is defined as
        \[
            \vaemask(\im[][]) = \dec(\neg\maski \odot \enc(\im[][])).
        \]
\end{enumerate}

Note that this overall structure generally follows the seminal work of \cite{de2019causal}. Our key contribution is Algorithm~\ref{alg: mask}, which provides a mask for the disentangled observations without relying on expert queries, the expert reward function, or specification of the causal graph. A visualization of Algorithm~\ref{alg: mask} is provided in Figure~\ref{fig: dependencies} for the \cp system considered in the experiments. We show in Section~\ref{sec: theory} that Algorithm~\ref{alg: mask} enjoys notable theoretical guarantees.

\begin{algorithm}
	\caption{Masking algorithm} \label{alg: mask} 
    \hspace{2mm} Hyperparameter $\gamma > 0$.

\begin{algorithmic}
\Procedure{Mask}{$\trajs$}
	\State Initialize $\maski \in \{0, 1\}^{\obdim}$ to be an all-zero vector.
	\For{$\oi = 1, \ldots, \obdim$}
    	\State Mask the $\oi^{\text{th}}$ observation according to
    	\vspace{-1.5mm}
    	\begin{small}
    	\begin{align} \label{eqn: mask}
            \quad \maski_{\oi} \gets \left( \Ob[1][\oi] \npcau \Ac[t'][\ai] \;\; \forall \ai \in [\acdim], \;\; \forall t' \in [H] \right), \\[-6.5mm] \nonumber
    	\end{align}
    	\end{small}
        \hspace{8mm} computing $\Ob[1][\oi] \npcau \Ac[t'][\ai]$ using \textsc{Check}. 
	\EndFor
	\State \textbf{return} $\maski$
\EndProcedure
\end{algorithmic}
\vspace{0.2cm}
\begin{algorithmic}
\Procedure{Check$\{\Ob[t][\oi] \pcau \Ac[t'][\ai]\}$}{$\trajs$} 
	\For{$\si = 1, \dots, \stdim$}
		\State $a \gets \textsc{Hoeffding}(\St[t][\si] \nPI \Ob[t][\oi]\ \mathrm{in}\ \trajs) > \gamma$
		\State $b \gets \textsc{Hoeffding}(\St[t][\si] \nPI \Ac[t'][\ai]\ \mathrm{in}\ \trajs) > \gamma$
		\If{$a \land b$}
			\State {\bf return} True
		\EndIf
	\EndFor
	\State {\bf return} False
\EndProcedure
\end{algorithmic}
\end{algorithm}

\section{THEORETICAL GUARANTEES} \label{sec: theory}

% This section analyzes the theoretical properties of Algorithm~\ref{alg: mask}. Theorem~\ref{thm: conservative} shows that, under certain conditions in the infinite-trajectory regime, distributionally intervening the initial state $\St[1]$ maintains \emph{conservativeness}: the algorithm will never incorrectly mask an observation that causally influences the expert. Theorem~\ref{thm: no_more_conservative} and Proposition~\ref{prop: strictly_less_conservative} illustrate that intervening $\St[1]$ reduces the potential overconservativeness in the masking algorithm. Specifically, Theorem~\ref{thm: no_more_conservative} states that if an observation $\Ob[][\oi]$ would be correctly masked under the causal model $\sys$, intervening $\St[1]$ will produce a new causal model $\sysi$ under which $\Ob[][\oi]$ is also masked. Proposition~\ref{prop: strictly_less_conservative} provides an illustrative class of systems for which an observation will only be masked in the intervened causal model $\sysi$, showing that intervening on $\St[1]$ produces a strictly better mask.
In this section, we delve into the theoretical properties of Algorithm~\ref{alg: mask}. Theorem~\ref{thm: conservative} demonstrates that if we intervene on the initial state $\St[1]$ and meet certain conditions in the infinite-trajectory regime, the algorithm remains \emph{conservative}, ensuring that no observation that causally influences the expert is mistakenly masked. Additionally, Theorem~\ref{thm: no_more_conservative} and Proposition~\ref{prop: strictly_less_conservative} highlight the effectiveness of intervening on $\St[1]$ in mitigating overconservativeness in the masking algorithm. Specifically, Theorem~\ref{thm: no_more_conservative} asserts that the correctly masked observations under the original causal model $\sys$ will also be masked under the intervened causal model $\sysi$. Proposition~\ref{prop: strictly_less_conservative} showcases a particular set of systems where the intervention only results in masks under $\sysi$, providing compelling evidence that the masking algorithm is more effective after intervening on $\St[1]$.

\begin{toappendix}
    We first introduce a series of auxiliary lemmas.
\begin{lemma} \label{lem: faithful}
    Consider an interventionally absolutely continuous SCM $\M$ with a faithful causal graph $\G$ that contains a directed path from $X$ to $Y$. Then provided a set $\bZ$ contains all ancestors of $X$ but none of its descendants, then for any assignment $\bz$ to $\bZ$ there exist values $x, x'$ such that
    \[
        \Pnorm{\Pd{y \mid \doc(x), \bz} - \Pd{y \mid \doc(x'), \bz}} > 0,
    \]
    viewed as induced measures over $Y$.
\end{lemma}

\begin{proof}
    As $\bZ$ contains no descendants of $X$, it cannot block the directed path between $X$ and $Y$ and hence the Causal Markov Condition does not declare $X$ and $Y$ independent. Faithfulness stipulates that $X$ and $Y$ are therefore dependent given $\bz$, so there exist $x,x'$ such that
    \[
        \Pnorm{\Pd{y \mid x, \bz} - \Pd{y \mid x', \bz}} > 0.
    \]

    The second rule of do calculus states that we can exchange observation and intervention if $X$ and $Y$ are independent given $\bz$ in the causal graph $\G_{\underline{X}}$ obtained by removing outgoing edges from $X$. If we remove outgoing edges from $X$, the only remaining paths between $X$ and $Y$ must contain an edge $X \leftarrow Z$ for some variable $Z$. This makes $Z$ an ancestor of $X$, and therefore $Z$ is included in $\bZ$, and both paths of the form $X \leftarrow Z \leftarrow J$ and $X \leftarrow Z \cau J$ are blocked by $\bZ$. This means that $X$ and $Y$ are $d$-separated by $\bZ$ in $\G_{\underline{X}}$, and we can apply the second do-calculus rule to conclude that
    \begin{align*}
        \Pd{y \mid \doc(x), \bz} =\ &\Pd{y \mid x, \bz}, \\
        \Pd{y \mid \doc(x'), \bz} =\ &\Pd{y \mid x', \bz},
    \end{align*}
    and hence
    \[
        \Pnorm{\Pd{y \mid \doc(x), \bz} - \Pd{y \mid \doc(x'), \bz}} > 0.
    \]
\end{proof}

\begin{lemma} \label{lem: measure_zero}
    Consider a set $E \subseteq \R$ where for each $x \in E$, there exists a ball $B(x, \epsilon_x)$ which contains no point in $E$. Then $E$ has measure zero with respect to the standard Lebesgue measure on $\R$.
\end{lemma}

\begin{proof}
    As $E$ is a subset of $\R$, it is Lindel{\"o}f, and the cover of $E$ by the collection of balls $\{B(x, \epsilon_x) \mid x \in E\}$ has a finite subcover. Enumerate this subcover as $I_i$; we then have
    \begin{align*}
        \lambda(E) &= \lambda(E \cap (\cup_i I_i)) \leq \sum_i \lambda(E \cap I_i) = 0,
    \end{align*}
    as each $E \cap I_i$ contains only a singleton.
\end{proof}

\begin{lemma} \label{lem: deriv_abs_value_swap}
    Let $f(x)$ be a differentiable function of $x \in \R$ at some $\bar x \in \R$ with $f(\bar x) = 0$. Then 
\[
\frac{d}{dx} \biggr \rvert_{\bar{x}^+} |f(x)| = \Bigg| \frac{d}{dx} \biggr \rvert_{\bar{x}} f(x) \Bigg|
\quad \mathrm{and} \quad
\frac{d}{dx} \biggr \rvert_{\bar{x}^-} |f(x)| = - \Bigg| \frac{d}{dx} \biggr \rvert_{\bar{x}} f(x) \Bigg|.
\]
\end{lemma}
\begin{proof}
    We prove the first result as the second follows similarly. Expanding the derivative:
    \begin{align*}
    \frac{d}{dx} \biggr \rvert_{\bar{x}^+} |f(x)|
    &= \lim_{\delta \to 0^+} \frac{|f(\bar x + \delta)| - |f(\bar x)|}{\delta} \\
    &= \lim_{\delta \to 0^+} \frac{|f(\bar x + \delta) - f(\bar x)|}{\delta} \\
    &= \left| \lim_{\delta \to 0^+} \frac{f(\bar x + \delta) - f(\bar x)}{\delta} \right| \\
    &= \Bigg| \frac{d}{dx} \biggr \rvert_{\bar{x}} f(x) \Bigg|,
    \end{align*}
    where moving the limit inside the absolute value is permissible by differentiability of $f$ at $\bar x$ and continuity of absolute value.
\end{proof}

\begin{lemma} \label{lem: deriv_swap}
    Let $f(x): \R \to M(Y)$ continuously map real numbers $x$ to a measure over the values assumed by a random variable $Y$. Then we have that
    \[
        \frac{d}{db} \biggr \rvert_{\bar b} \frac{d}{d \lambda} \left( \int_a^b f(x) dx \right)
        = \frac{d}{d \lambda} f(\bar b)
    \]
    almost everywhere over the domain of $Y$. Here $\int_a^b f(x) dx$ denotes Lebesgue integration against $\cU(a,b)$, $\frac{d}{d \lambda}$ is the Radon-Nikodym derivative with respect to the standard Lebesgue measure on $\R$, and $\frac{d}{db} \bigr \rvert_{\bar b}$ denotes the standard real analysis derivative evaluated at $\bar b$.
\end{lemma}
\begin{proof}
    We can expand the definition of the outer derivative
    \begin{align*}
        \frac{d}{db} \biggr \rvert_{\bar b} \frac{d}{d \lambda} \left( \int_a^b f(x) dx \right)
        &= \lim_{b \to \bar b} \frac{1}{b - \bar b} \left(\frac{d}{d \lambda} \left( \int_a^{\bar b} f(x) dx \right) - \frac{d}{d \lambda} \left( \int_a^b f(x) dx \right) \right) \\
        &= \lim_{b \to \bar b} \frac{1}{b - \bar b} \left(\frac{d}{d \lambda} \left( \int_b^{\bar b} f(x) dx \right) \right). \\
        &= \lim_{b \to \bar b} \frac{d}{d \lambda} \left( \frac{1}{b - \bar b} \int_b^{\bar b} f(x) dx \right).
    \end{align*}
    Take an arbitrary $\epsilon > 0$. We want to show $\exists \delta > 0$ such that for all $\bar b - \delta < b < \bar b + \delta$, we have that
    \begin{align*}
        \left\|  \frac{1}{b - \bar b} \left( \int_b^{\bar b} f(x) dx \right) - f(\bar b) \right\|_1 < \epsilon,
    \end{align*}
    where the Radon-Nikodym derivative $\frac{d}{d \lambda}$ is absorbed into the $L_1$ norm definition on measures. By continuity of $f$, we can always choose a $\delta$ small enough for this inequality to hold.
\end{proof}

\vspace*{1cm}

We now prove the main theoretical results.

\end{toappendix}

All subsequent theory relies on Assumptions~\ref{ass: time_invariance}-\ref{ass: abs_cont}, and for brevity we defer proofs and auxiliary lemmas to the appendix. We now introduce the main conservativeness theorem and provide a short proof sketch.

\begin{theoremrep} \label{thm: conservative}
    In the faithful system causal model $\sys$, assume that the measure-valued function ${\hi[1] \mapsto \Pdi{v \mid \doc(\bZ = \bz), \hi[1]}}$ is continuous for any set of nodes $\bZ$ and $V \not \in \bZ$.

    Let there exist a causal edge ${\Ob[t][\oi] \cau \Ac[t'][\ai]}$ in $\Gs$ for some $t,t' \in \N$, $t' \geq t$, and indices $\oi \in [\obdim]$ and $\ai \in [\acdim]$. Then in the interventional causal model $\sysi$ where the initial state distribution $\Pdint{\st[1]}$ has everywhere-nonzero density on $\stspace$, $\Ob[][\oi]$ is almost surely not masked by Algorithm~\ref{alg: mask} for almost every uniform parameterization of $\Hi[1]$ as the number of trajectories $N \to \infty$; i.e., \eqref{eqn: mask} correctly evaluates to true.
\end{theoremrep}
\begin{customproofsketch}
% How to handle almost all uniform parameterizations? With respect to what measure?
By Assumptions~\ref{ass: time_invariance}~and~\ref{ass: reaction_horizon}, we can WLOG consider $t=1$ with $t' \in [H]$. If $\Ob[1][\oi] \cau \Ac[t'][\ai]$, by Assumption~\ref{ass: expert_attend} there exists an edge $\St[1][\si] \to \Ob[1][\oi]$ for some $\si$. We show that in the SCM $\Msi$ where we intervene distributionally on $\St[1]$, we have that $\St[1][\si] \nPI \Ob[1][\oi]$ and $\St[1][\si] \nPI \Ac[t'][\ai]$. The arguments are similar, so we informally sketch the proof for the former.

To show that $\St[1][\si]$ and $\Ob[1][\oi]$ are dependent, it suffices to find a particular pair of states $\St[1][\si] = \alpha, \alpha'$ which induce different probability measures $\Pd{\ob[1][\oi] \mid \doc(\St[1][\si] = \alpha)}$ (resp. $\alpha'$) over $\Ob[1][\oi]$. 
We marginalize out the random seed $w$ from our original measure of interest $\Pd{\ob[1][\oi] \mid \doc(\St[1][\si] = \alpha)}$ via the integral
\begin{align*}
\Pd{\ob[1][\oi] \mid \doc(\St[1][\si] = \alpha)} = \int_a^b \Pd{\ob[1][\oi] \mid \doc(\St[1][\si] = \alpha), \hi[1]} p(w) d \hi[1],
\end{align*}
where we model $\hi \sim \cU(a,b)$. Note that the right-hand integral above in fact yields a measure over $\ob[1][\oi]$. We now aim to show that the statement
\begin{align}
\begin{aligned}
    \exists \alpha, \alpha' \textrm{ s.t. }
    \Pnorm{\int_a^b \big[&\Pd{\ob[1][\oi] \mid \doc(\St[1][\si] = \alpha), \hi[1]} - \\
                        &\Pd{\ob[1][\oi] \mid \doc(\St[1][\si] = \alpha'), \hi[1]} \big] d \hi} > 0 \label{eq: sketch2} 
\end{aligned}
\end{align}
holds Lebesgue-almost everywhere for $(a,b) \in \R^2$. By faithfulness of $\sys$ and the path from $\St[1][\si]$ to $\Ob[1][\oi]$, do-calculus rules yield that for any random seed $\Hi=\hi$ there exist an $\alpha, \alpha'$ such that
\begin{align} \label{eq: sketch1}
    \Pnorm{\Pd{\ob[1][\oi] \mid \doc(\St[1][\si] = \alpha), \hi[1]} - \Pd{\ob[1][\oi] \mid \doc(\St[1][\si] = \alpha'), \hi[1]}} > 0.
\end{align}
We then analyze the sensitivity of \eqref{eq: sketch2} with respect to the integration bounds $a$ and $b$. Namely, for any $(\bar a, \bar b)$ where the left-hand side of \eqref{eq: sketch2} vanishes, \eqref{eq: sketch1} yields that there exists an open ball around $\bar b$ in which \eqref{eq: sketch2} holds everywhere except $(\bar a, \bar b)$. An argument from Fubini's theorem then shows that \eqref{eq: sketch2} holds for almost all $(a,b)$. Appealing to the consistency of Hoeffding's independence test concludes the proof.
\end{customproofsketch}

\begin{appendixproof}
    By Assumptions~\ref{ass: time_invariance}~and~\ref{ass: reaction_horizon}, we can WLOG consider $t=1$ with $t' \in [H]$. If $\Ob[1][\oi] \cau \Ac[t'][\ai]$, by Assumption~\ref{ass: expert_attend} there exists an edge $\St[1][\si] \to \Ob[1][\oi]$ for some $\si$. We now want to show that in the SCM $\Msi$ where we intervene distributionally on $\St[1]$, we have that $\St[1][\si] \nPI \Ob[1][\oi]$ and $\St[1][\si] \nPI \Ac[t'][\ai]$. The arguments are similar, so we will just state the proof for the former.

    We want to show that $\St[1][\si]$ and $\Ob[1][\oi]$ are not independent in $\Msi$. Note that in the modified structural assignment for $\St[1][\si]$ in $\Msi$, $\St[1][\si]$ is distributed with everywhere-nonzero density on $\stspace$. Therefore checking the desired independence is equivalent to showing
    \begin{align} \label{eqn: conservative_inequality_1}
        \Pnorm{\Pd{\ob[1][\oi] \mid \doc(\St[1][\si] = \alpha)} - \Pd{\ob[1][\oi] \mid \doc(\St[1][\si] = \alpha')}} > 0
    \end{align}
    for some $\alpha, \alpha' \in \R$ with $\alpha \neq \alpha'$. Here, $\Pd{\ob[1][\oi] \mid \cdot}$ denotes a probability measure over $\ob[1][\oi]$. The $\doc$ statement captures our ability to intervene on the initial state, decoupling any potential correlational influence from $\Hi[1]$.

    By Lemma~\ref{lem: faithful}, we have that for any particular value $\hi[1]$ of $\Hi[1]$,
    \begin{align*}
        \Pnorm{\Pd{\ob[1][\oi] \mid \doc(\St[1][\si] = \alpha), \hi[1]} - \Pd{\ob[1][\oi] \mid \doc(\St[1][\si] = \alpha'), \hi[1]}} > 0,
    \end{align*}
    for some $\alpha$, $\alpha'$. This is equivalent to
    \begin{align}
         \Pnormi{h(\alpha, \alpha', \hi[1])} \not\equiv \zf \quad \forall \hi[1],  \label{eqn: h_nonzero}
    \end{align}
    where we define
    \begin{align*}
        h(\alpha, \alpha', \hi[1]) \defeq \Pd{\ob[1][\oi] \mid \doc(\St[1][\si] = \alpha), \hi[1]} - \Pd{\ob[1][\oi] \mid \doc(\St[1][\si] = \alpha'), \hi[1]},
    \end{align*}
    and $\zf$ denotes an identically zero function over $\alpha, \alpha'$. Note that $h(\alpha, \alpha', \hi[1])$ specifies a signed measure over $\ob[1][\oi]$. Now observe that
    \begin{align*}
    \Pd{\ob[1][\oi] \mid \doc(\St[1][\si] = \alpha)} = \int \Pd{\ob[1][\oi] \mid \doc(\St[1][\si] = \alpha), \hi[1]} p(\hi[1]) d\mu(\hi[1]),
    \end{align*}
	where $\mu$ is a probability measure on the unobserved variable $\hi[1]$ which we will instantiate shortly, and $p(\hi[1])$ denotes the probability density of $\Hi[1]$, i.e. the Radon-Nikodym derivative of the measure $\Pdi{\hi[1]}$. Note that the result of this integral is still a signed measure over $\ob[1][\oi]$. So we have that showing our desired inequality \eqref{eqn: conservative_inequality_1} is equivalent to showing
    \begin{align*}
        \Pnorm{\int h(\alpha, \alpha', \hi[1]) p(\hi[1]) d\mu(\hi[1])} \not\equiv \zf
    \end{align*}
    as a function of $\alpha$, $\alpha'$ for ``almost all'' measures $\mu$---as there is no natural measure on the space of measures, we have formalized this assuming a uniform distribution $\Hi[1] \sim \cU(a,b)$. Note that the outer norm computes the $L_1$ norm of a signed measure over $\ob[1][\oi]$. For notational convenience, we will now define the concatenation $z = [\alpha, \alpha']$, with $z \in \R^2$. We can now concretely refine $\mu$ in the above statement, using our new $z$-notation, to showing that
    \begin{align}
        g_{a}^b(z) \defeq \Pnorm{\int_a^b h(z, \hi[1]) d\hi[1]} \not\equiv \zf \label{eqn: conservative_inequality_2} 
    \end{align}
    as a function of $z$ for almost every $(a,b)$; i.e., the subset of $(a,b)$ parameter space where $g_a^b(z) \equiv \zf$ over $z$ is measure zero with respect to the standard Lebesgue measure in $\mathbb{R}^2$. Note that we drop the $p(\hi[1])$ factor since for the uniform distribution this is a constant which factors out.

    This can be analyzed by taking sections were we fix $a$ and consider the set of $b$'s where $g_a^b (z) \equiv \zf$; if this set has measure zero, then the overall set of Cartesian pairs $(a,b)$ where $g_a^b(z) \equiv \zf$ can be shown to have measure zero by the following argument. Observe that $g_a^b(z)$ is continuous in $a, b$ and $z$ by Assumption~\ref{ass: abs_cont} and integral properties; then the inverse image of $\{0\}$ under $g$ is a Borel subset of $A \subset \R^4$, recalling that $z \in \R^2$. The projection of this Borel subset on to the $(a,b)$ plane is measurable (but not necessarily Borel). Then if each fixed-$a$ slice is measure zero, the overall set is measure zero by Fubini's theorem.

    Correspondingly, we fix any particular $a$ and drop it from the subscript of $g_a^b$ for simplicity. Consider a particular $\bar b$ where $g^{\bar b}(z) \equiv \zf$ as a function of $z$. We expand the $L_1$ norm in \eqref{eqn: conservative_inequality_2} as
    \begin{align}
        g^b(z) = \int \biggl \lvert \frac{d}{d\lambda} \bigg( \int_a^b h(z, \hi[1]) d\hi[1] \bigg) \biggr \rvert d \lambda
        = \int \left| f^b_z(\ob[1][\oi]) \right| d \lambda, \label{eqn: deriv_expansion}
    \end{align}
    using the interventional absolute continuity assumption to invoke the Radon-Nikodym derivative on our signed measure over $\ob[1][\oi]$ with respect to the standard Lebesgue measure $\lambda$. We denote the resulting density function by $f^b_z(\ob[1][\oi])$. Note that since $g^{\bar b}(z) \equiv \zf$, we have that $f^{\bar b}_z(\ob[1][\oi]) = 0$ for almost all $z$ and $\ob[1][\oi]$.

    Note that $f^b_z(\ob[1][\oi])$ is differentiable with respect to $b$ due to the assumed continuity of maps on $\hi[1]$ in the theorem statement. We now differentiate \eqref{eqn: deriv_expansion} with respect to $b$ at $\bar b$. Due to the absolute value in \eqref{eqn: deriv_expansion}, we must take care to differentiate from above and below and show both these cases are nonzero. As they follow similarly, we show the case for above:
    \begin{align}
		\frac{d}{db} \biggr \rvert_{\bar{b}^+} g^b(z) &= \frac{d}{db} \biggr \rvert_{\bar b^+} \int \biggl \lvert \frac{d}{d\lambda} \bigg( \int_a^b h(z, \hi[1]) d\hi[1] \bigg) \biggr \rvert d \lambda \label{eqn: step1} \\
 		&= \int \frac{d}{db} \biggr \rvert_{\bar b^+} \biggl \lvert \frac{d}{d\lambda} \bigg( \int_a^b h(z, \hi[1]) d\hi[1] \bigg) \biggr \rvert d \lambda \label{eqn: step2}  \\
 		&= \int \biggl \lvert \frac{d}{db} \biggr \rvert_{\bar b} \frac{d}{d\lambda} \bigg( \int_a^b h(z, \hi[1]) d\hi[1] \bigg) \biggr \rvert d \lambda \label{eqn: step3}  \\
 		&= \int \biggl \lvert \frac{d}{d\lambda} h(z, \bar b) \biggr \rvert d \lambda \label{eqn: step4}  \\
 		&= \Pnorm{h(z, \bar b)} \label{eqn: step5}  \\
 		&\not\equiv \zf, \qquad\qquad \text{(as a function of $z$)} \label{eqn: step6}
    \end{align}
    where \eqref{eqn: step2} follows from boundedness of the Radon-Nikodym derivative of $h(z, \bar b)$, \eqref{eqn: step3} follows from applying Lemma~\ref{lem: deriv_abs_value_swap} to $f^b_z(\ob[1][\oi])$ with respect to $b$, \eqref{eqn: step4} follows from Lemma~\ref{lem: deriv_swap}, \eqref{eqn: step5} follows from differentiability of $f_z^b(\ob[1][\oi])$ with respect to $b$, and \eqref{eqn: step6} follows from \eqref{eqn: h_nonzero}.

    Proceeding similarly, we can show that both
    \begin{align*}
    \frac{d}{db} \biggr \rvert_{\bar{b}^+} g^b(z) \not \equiv \zf \quad \mathrm{and} \quad
    \frac{d}{db} \biggr \rvert_{\bar{b}^-} g^b(z) \not \equiv \zf.
    \end{align*}
    It is then immediate that there exists a ball $B(\bar b, \epsilon_{\bar b})$ such that $g^b(z) \not \equiv \zf$ for all $b \in B(\bar b, \epsilon_{\bar b}) \setminus \bar b$. Applying Lemma~\ref{lem: measure_zero} concludes that for a fixed $a$, the set of $b$ for which \eqref{eqn: conservative_inequality_2} is violated is measure zero. Hence by the above Fubini argument, for almost every uniform measure $\mathcal{U}(a,b)$ on $\hi[1]$, we have that \eqref{eqn: conservative_inequality_1} holds for some $\alpha, \alpha'$. Therefore $\St[1][\si] \nPI \Ob[1][\oi]$ in the interventional distribution on $\St[1][\si]$.

    A similar argument shows that $\St[1][\si] \nPI \Ac[t'][\ai]$. By absolute continuity of the induced interventional distributions, we now have that Hoeffding's independence test is consistent, and hence the dependences are detected with probability $1$ as $N \to \infty$. Therefore $\Ob[1][\oi] \pcau \Ac[t'][\ai]$, and $\maski_{\oi}$ evaluates to false \eqref{eqn: mask} as $N \to \infty$. 
    % TODO: cite consistency
\end{appendixproof}

% Theorem~\ref{thm: conservative} provides the conservativeness guarantee that we expected: if an observation causally affects some expert action, Algorithm~\ref{alg: mask} correctly keeps this observation unmasked. As discussed in Section~\ref{sec: derivation}, this is immediate by the faithfulness of $\sys$ when we do not intervene on $\St[1]$ and allow the initial state to be generated naturally from $\Hi[1]$. The effort of Theorem~\ref{thm: conservative} is devoted to showing that this property still holds in the interventional system $\sysi$ where we assign $\St[1] \sim \Pdint{\st[1]}$.
Theorem~\ref{thm: conservative} guarantees that Algorithm~\ref{alg: mask} maintains conservativeness by correctly preserving unmasked observations that causally impact expert actions. This outcome is consistent with the discussion in Section~\ref{sec: derivation}, where we observed that the faithfulness of $\sys$ ensures the correctness of the algorithm when we do not intervene on $\St[1]$ and allow the initial state to be naturally generated from $\Hi[1]$. Theorem~\ref{thm: conservative} establishes that this property also holds in the interventional system $\sysi$, where we assign $\St[1] \sim \Pdint{\st[1]}$.

We now theoretically demonstrate the benefits of intervening with $\Pdint{\st[1]}$. Specifically, we show that this intervention reduces the excess conservatism in the masking algorithm by removing income edges from $\Hi[1]$ in the causal graph, thereby eliminating a potential avenue of confounding.

\begin{theoremrep} \label{thm: no_more_conservative}
    Let $\mask$ denote the potential-cause test evaluated by Algorithm~\ref{alg: mask} on the distribution induced by the non-interventional system $\sys$, and let $\maski$ be the original test on the interventional system $\sysi$ where $\Pdint{\st[1]}$ has everywhere-nonzero density on $\stspace$. If $\mask_{\oi}$ correctly evaluates to true for a particular $\oi \in [\obdim]$, then $\maski_{\oi}$ also evaluates to true almost surely as the number of trajectories $N \to \infty$.
\end{theoremrep}
\begin{appendixproof}
    If $\mask_{\oi}$ evaluates to true, then for any $\si \in [\stdim]$, $\ai \in [\acdim]$, and $t' \in [H]$, we have that either $\St[1][\si] \PI_{\Ms} \Ob[1][\oi]$ or $\St[1][\si] \PI_{\Ms} \Ac[\ai][t']$, where $\PI_{\Ms}$ denotes independence in the distribution induced by the non-interventional SCM $\Ms$. It suffices to show that both these independencies hold in the distribution induced by $\Msi$. As both arguments follow similarly, we consider showing that $\St[1][\si] \PI_{\Msi} \Ob[1][\oi]$.

    As we are given $\St[1][\si] \PI_{\Ms} \Ob[1][\oi]$, it is immediate by faithfulness that there exists no collider-free path from $\St[1][\si]$ to $\Ob[1][\oi]$ in $\Gs$. Since $\Gsi$ is simply $\Gs$ with the incoming edges to $\St[1]$ removed, it holds that there is no collider-free path between $\St[1][\si]$ and $\Ob[1][\oi]$ in $\Gsi$. Therefore $\St[1][\si] \PI_{\Msi} \Ob[1][\oi]$, and as $N \to \infty$ this is correctly detected with probability $1$ by the consistency of Hoeffding's test.
\end{appendixproof}

Theorem~\ref{thm: no_more_conservative} assures us that intervening with $\Pdint{\st[1]}$ does not lead to more conservative masking than the original system. We now provide a specific class of SCMs for which the intervention strictly improves the mask.

\begin{propositionrep} \label{prop: strictly_less_conservative}
    Let $\maski$ and $\mask$ be as in Theorem~\ref{thm: no_more_conservative}, and consider a particular observation index $\oi \in [\obdim]$ such that the only incoming edge to $\Ob[1][\oi]$ is $\Hi[1] \cau \Ob[1][\oi]$. Then if in $\Gs$ there exists the fork $\St[1][\si] \leftarrow \Hi[1] \cau \Ob[1][\oi]$ for some $\si \in [\stdim]$ and a directed path from $\St[1][\si]$ to some $\Ac[t][\ai]$, with $t \in [H], \ai \in [\acdim]$, $\maski_{\oi}$ correctly masks the $\oi \tth$ observation almost surely as the number of trajectories $N \to \infty$ while $\mask_{\oi}$ does not.
\end{propositionrep}

\begin{appendixproof}
	We first show that $\mask$ does not mask $\oi$ and take all causal and probabilistic statements to refer to the unintervened causal model $\sys$. By faithfulness, the fork ${\St[1][\si] \leftarrow \Hi[1] \rightarrow \Ob[1][\oi]}$ in $\Gs$ produces a statistical dependence $\St[1][\si] \,{\nPI_{\Ms}}\, \Ob[1][\oi]$ in the probability distribution induced by $\Ms$. Similarly, the directed path from $\St[1][\si]$ to $\Ac[t][\ai]$ yields ${\St[1][\si]\, {\nPI_{\Ms}}\, \Ac[t][\ai]}$. By consistency of Hoeffding's test, as ${N \to \infty}$ we get that ${^{(1,t)} D_{\si,\ai}^{\oi}}$ evaluates to true almost surely \eqref{eqn: causal_check} and thus ${\Ob[1][\oi] \pcau \Ac[t][\ai]}$ by \eqref{eqn: potential_cause}. Therefore $\mask_{\oi}$ is not masked \eqref{eqn: mask}.

    We now show that $\maski$ does mask $\oi$ and take all causal and probabilistic statements to refer to the \emph{intervened} causal model $\sysi$. Since $\Hi[1]$ only has outgoing edges, and the edge from $\Hi[1] \to \St[1][\si']$ is removed in $\Gsi$ for every $\si' \in [\stdim]$, there exists no path from $\St[1][\si']$ to $\Ob[1][\oi]$ in $\Gsi$, and therefore $\St[1][\si'] \PI_{\Msi} \Ob[1][\oi]$ in the probability distribution induced by $\Msi$. As $N \to \infty$ this independence is detected by Hoeffding's test, and since $\si'$ was arbitrary ${^{(1,t')} D_{\si',\ai'}^{\oi}}$ is false for every $\si' \in [\stdim]$, $\ai' \in [\acdim]$, and $t' \in [H]$. Therefore $\Ob[1][\oi] \npcau \Ac[t'][\ai']$ for any $\ai' \in [\acdim],t' \in [H]$, and \eqref{eqn: mask} evaluates to true. Therefore $\maski_{\oi}$ is masked.
\end{appendixproof}

% In summary, Theorem~\ref{thm: conservative} shows that masking with the intervened initial state $\Pdint{\st[1]}$ maintains conservatism; Theorem~\ref{thm: no_more_conservative} states that intervening on $\Pdint{\st[1]}$ is no more conservative than masking with the unintervened causal model; and Proposition~\ref{prop: strictly_less_conservative} shows that intervening on $\Pdint{\st[1]}$ results in a strictly less conservative mask for a certain class of systems.

\section{EXPERIMENTS} \label{sec: experiments}
\begin{figure*}[t] 
    \centering
    \input{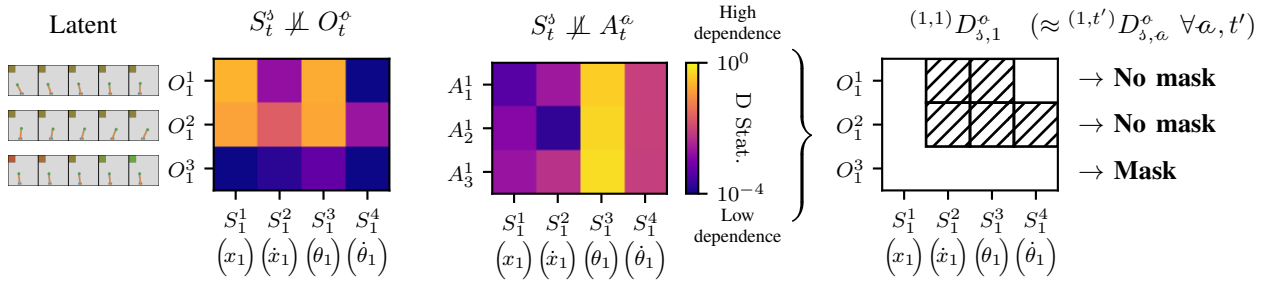}
    \caption{Masking algorithm visualization for the \cp environment with reaction horizon $H=3$. Latent space interpolation of the $\beta$-VAE reveals that $\Ob[][1]$ and $\Ob[][2]$ capture some combined positional/angular information, while $\Ob[][3]$ captures the disentangled confounder (color of the confounding square). This last observation shares virtually no dependence (Hoeffding's D statistic less than $\gamma=10^{-3}$) with any state variable due to interventions on $\St[1]$ (note the log scale). This means that ${^{(1,t')} D_{\si,\ai}^{\oi}}$ is false (no cross hatches) for $\oi=3$ and all $\si \in [\stdim]$, regardless of $\ai$ and $t'$; i.e. $\Ob[1][3] \npcau \Ac[t'][\ai]$ for all $\ai \in [\acdim]$ and $t' \in [H]$, and we can mask the confounder $\Ob[][3]$. \vspace*{-0.27cm}}
    \label{fig: dependencies}
\end{figure*}

We evaluate our approach on two custom simulated environments: \cp and \re. Each of these environments contains a nuisance feature which is likely to induce causal confusion. Our masking approach can successfully eliminate these spuriously correlated features. Precise experimental details are deferred to Appendix~\ref{app: experiments}.

\subsection{Environments} \label{sec: env}

Both considered environments are modified to include a nuisance feature corresponding to the previous action taken by the expert (analogous to the brake light example). For each environment, the expert is a standard constrained finite-time optimal control policy which minimizes cumulative trajectory loss. This expert reward function is not provided to the imitation learning agent.

\textbf{\cp}. This environment consists of a standard planar cart-pole system with a continuous scalar horizontal force applied to the cart. A quadratic cost is imposed for deviations from the vertical target state. The spuriously correlated feature is a colored square in the upper-left corner of each image, which interpolates between green and red depending on the most recently executed action.

\textbf{\re}. We consider a top-down version of a two-dimensional two-joint \re environment \cite{OpenAIGym}. The environment penalizes squared distance of the end effector to a black target dot. The target location is included in the state vector, thus satisfying Assumption~\ref{ass: expert_attend}. Two torques, one per joint, are specified as the control inputs; the nuisance feature is a red dot in the upper-left corner whose horizontal position and vertical position encode the two control inputs from the previous time step. This ``joystick'' introduces a different kind of nuisance feature than in the \cp environment.

\subsection{Discussion}

We compare the performance of our masked policy against vanilla behavior cloning. The baseline behavior cloning policy is denoted by \bc, and our masked policy is denoted by \masked. For reference, we also measure the performance of the behavior cloning policy with the confounding signals manually removed by superimposing a white square on the upper-left corner, denoted \bcd. We emphasize that \bcd requires human judgement to manually eliminate spurious confounders; we show that we can approach this performance in a principled and automated way.

Figure~\ref{fig: results} displays our experimental results. For \cp, the policies were not able to consistently stabilize the pendulum at the beginning of training, leading to high loss variance. Across both environments, the \masked policy substantially outperforms the vanilla behavior cloning policy $\bc$. It is worth noting that \masked approaches the manually deconfounded baseline's performance without requiring expert queries, access to the expert reward function, or pre-specified information on the causal graph in the deconfounding procedure. However, there is a gap between the performance of our method and manual masking for the \re environment. This is likely attributable to imperfect disentanglement in the $\beta$-VAE, and we expect that our approach could benefit substantially from future research in disentangled representation learning.

Figure~\ref{fig: dependencies} provides a visualization of our masking procedure and the resulting mask for the \cp environment. Note that our algorithm masks the third observation $\Ob[][3]$, corresponding precisely to the manually masked confounding square. While we use a latent space size of three (the precise number of independent factors of variation) for visualization purposes, our masking procedure is fully functional for larger choices of the latent size. For \re, although there are $6$ factors of variation in each image, a larger latent space of size $12$ yielded superior disentanglement and reconstruction performance.

\begin{figure*}[t]
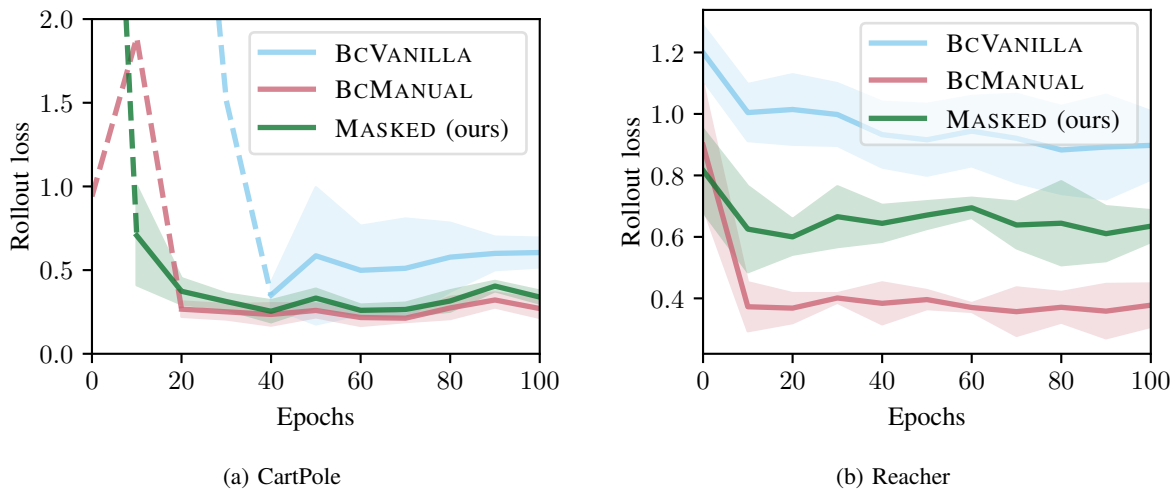

     \centering
     \begin{subfigure}[b]{0.45\textwidth}
         \centering
         \includegraphics{figs/cartpole_val_rollouts.pgf}
         \caption{\cp}
         \label{fig: results_cartpole}
     \end{subfigure}
     \begin{subfigure}[b]{0.45\textwidth}
         \centering
         \includegraphics{figs/reacher_val_rollouts.pgf}
         \caption{\re}
         \label{fig: results_reacher}
     \end{subfigure}
     \caption{Evaluation rollout loss on \cp (\subref{fig: results_cartpole}) and \re (\subref{fig: results_reacher}) across training epochs. Lines denote mean performance over $5$ runs while shaded areas indicate standard deviation. To limit visual clutter, for standard deviations greater than $1$ shading is omitted and the mean is drawn with a dashed line. Our \masked policy approaches the performance of the manually-deconfounded \bcd baseline, while \bc struggles due to causally confusing features. \vspace*{-0.27cm}}
        \label{fig: results}
\end{figure*}

The most significant limitation of our work, besides the explicitly stated assumptions, is the requirement that confounding factors are observable and can be neatly disentangled. While this holds for the environments considered in this work, more complex environments may introduce entanglement between causally confusing features and important features to which the expert policy actually attends. We introduce the Hoeffding threshold hyperparameter $\gamma$ to mitigate this concern; however, investigating more principled methods for handling incomplete disentanglement would be an exciting area of future work.

\section{CONCLUSION}
This work introduces a novel method to address the causal confusion problem in imitation learning. The proposed method leverages the typical imitation learning ability to intervene in the initial system state. Unlike previous works, our method masks causally confusing observations without relying on online expert queries, knowledge of the expert reward function, or specification of the causal graph. Our theoretical results establish that our masking algorithm is \emph{conservative}, with excess conservatism strictly reduced by interventions on the initial state. We illustrate the effectiveness of our method with experiments on \cp and \re.

\printbibliography

\AtEndDocument{\section{Experiments} \label{app: experiments}

We include here essential environment, architecture, and hyperparameter details.

\subsection{Environments} \label{app: environments}

We consider two environments: \cp and \re. Both systems are rendered to $64 \times 64$ RGB images, pictured in Figure~\ref{fig: environments}.

\begin{figure*}[h]
     \centering
     \begin{subfigure}[t]{0.4\textwidth}
         \centering
         \includegraphics[width=\textwidth]{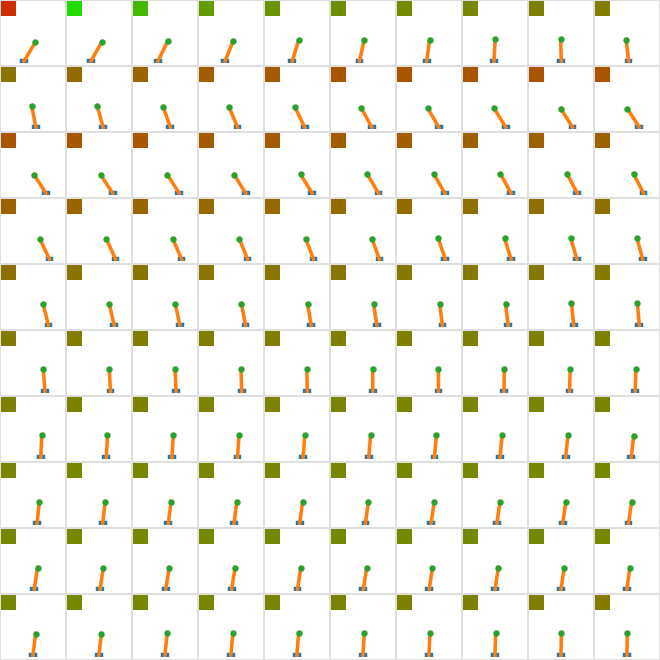}
         \caption{\cp}
         \label{fig: run_cartpole}
     \end{subfigure}
     \hspace*{1cm}
     \begin{subfigure}[t]{0.4\textwidth}
         \centering
         \includegraphics[width=\textwidth]{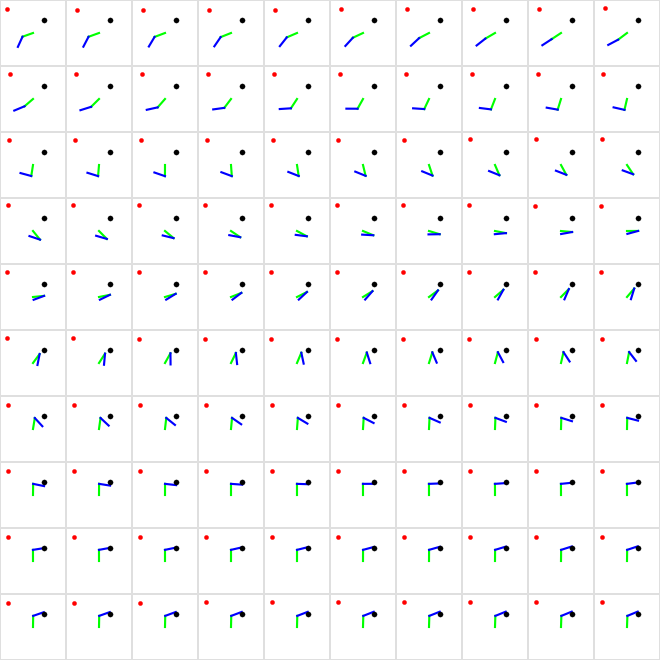}
         \caption{\re}
         \label{fig: run_reacher}
     \end{subfigure}
     \caption{Illustrative rollouts of the expert policy on \cp (\subref{fig: run_cartpole}) and \re (\subref{fig: run_reacher}). The \re run is truncated to $100$ time steps for visualization purposes. Note the nuisance feature in the upper left hand corner of the images.}
        \label{fig: environments}
\end{figure*}

\subsubsection{\cp}
The \cp environment \cite{barto1983neuronlike} describes a nonlinear dynamic system consisting of four states: the cart position $x$, the cart velocity $\dot{x}$, the pole angle $\theta$, and the pole angular velocity $\dot{\theta}$. The state vector at time $t$ is $S_t \coloneqq (\xt, \dxt, \tht, \dtht)$. The agent action is a continuous horizontal force acting on the cart, bounded symmetrically in the range $[-25, 25]$. The length of the pole is $1$ meter, with the masses of the cart and the pole set to $1$ and $0.1$ kilograms, respectively. We specify the gravitational acceleration constant as $g = 9.8 \textrm{m/s}^2$. The system is then discretized with $\Delta t = 0.05$ for $100$ time steps using the forward Euler method, with the standard CartPole dynamics equations adapted from OpenAI Gym \cite{OpenAIGym}.
We minimize cumulative stepwise quadratic form loss to the upright target state $\Starget = (0, 0, 0, 0)$, with an additional quadratic control cost.
%The initial state is sampled via $x_0 \sim \Unif(-2.5, 2.5)$, $\dot{x}_0 \sim \Unif(-5, 5)$, $\theta_0 \sim \Unif(-\frac{\pi}{4}, \frac{\pi}{4})$, $\dot{\theta}_0 \sim \Unif(-\frac{\pi}{4}, \frac{\pi}{4})$.
%The target state is upright state $\Starget = (0, 0, 0, 0)$, with a standard stepwise quadratic loss on deviation from this state and control cost. The initial state is sampled via $x_0 \sim \Unif(-2.5, 2.5)$, $\dot{x}_0 \sim \Unif(-5, 5)$, $\theta_0 \sim \Unif(-\frac{\pi}{4}, \frac{\pi}{4})$, $\dot{\theta}_0 \sim \Unif(-\frac{\pi}{4}, \frac{\pi}{4})$.
% The expert CFTOC policy additionally enforces that $\xt \in [-5, 5]$, $\dxt \in [-10, 10]$, $\tht \in [-\frac{\pi}{2}, -\frac{\pi}{2}]$, and $\dtht \in [-\frac{\pi}{2}, -\frac{\pi}{2}]$ for all $t$.

To each frame, we add a $15 \times 15$ square nuisance feature at the top-left corner of each image. The color of the square interpolates linearly between green and red, depending on the action (cart force) from the previous time step. At the initial time step $t = 1$ there is no previous action, and thus we use a random number drawn from $\Unif(-25, 25)$ to generate the square.

We generate $5,000$ random-policy trajectories for training the $\beta$-VAE and $1,000$ expert trajectories for imitation learning. Random-policy trajectories are terminated when the states become out-of-bound, and are thus generally significantly shorter than the expert trajectories.

\subsubsection{\re}
\re is also implemented based on the classic OpenAI Gym environment \cite{OpenAIGym}. The system contains six states: target position $x^*, y^*$; joint one angle and velocity $\theta_1, \dot \theta_1$; and joint two angle and velocity $\theta_2, \dot \theta_2$. The target positions $x^*, y^*$ is fixed over the course of one trajectory to a random point in the reachable area. Both links have mass $1$ kilogram and length $0.5$ meters. Agents specify torques at both joints, bounded in the range $[-2, 2]$. The objective penalizes squared distance of the end effector from $(x^*, y^*)$---visualized as a black dot---at each time step, along with a quadratic control cost. We simulate with a time step $\Delta t=0.05$ seconds for $200$ time steps.

For \re, we demonstrate that our method can eliminate a visually different type of confounding than the colored square in the \cp experiment. The considered confounder is a red dot in the upper-left corner of the image that moves translationally according to the agent action in the previous time step. Specifically, the horizontal position is linearly interpolated according to the first joint torque, and the vertical position is linearly interpolated according to the second joint torque. Similarly to \cp, we choose a random previous action for the first time step.

We generate $2,000$ random-policy trajectories for training the $\beta$-VAE and $2,000$ expert trajectories for imitation learning.

\subsection{VAE training}

We train a standard $\beta$-VAE \cite{burgess2018understanding} as implemented by \cite{Subramanian2020}. We choose a latent space dimension of $3$ for \cp and $12$ for \re, although we note that larger choices for the latent space dimension yield similar results. We train for $100$ epochs at a learning rate of $0.0005$ on \re and $0.005$ on \cp. Both use an exponential learning rate scheduler with decay factor $0.95$.  Our batch size is $256$ for \re and $64$ for \cp. Finally, we choose the disentaglement factor $\beta=100$ for \cp and $\beta=1000$ for \re.

\subsection{Behavior cloning training}

For all enviroments we use a standard pre-activation ResNet-18 \cite{he2016deep}. We train with the Adam optimizer \cite{kingma2014adam} at an initial learning rate of $0.001$ and exponential learning rate decay with factor $0.96$. We use a batch size of $256$ and evaluate the performance of the agent with $25$ validation rollouts every $10$ epochs. As a single image of the environment cannot convey higher-order state information such as velocity, we input the previous two images into our policies---i.e., $L=2$ in \eqref{eqn: policy}. Thus, we make the necessary architectural change to the underlying models of setting the number of input channels to $6$. Since the first time step does not have an associated previous image, we use a blank image as a surrogate.
}

\end{document}